\documentclass[letterpaper, 10 pt, journal, twoside]{IEEEtran}  

\usepackage{amsmath,amssymb,amsfonts,amsthm}
\usepackage{dsfont}
\usepackage{xfrac}
\usepackage{cite}
\usepackage{graphicx}
\usepackage{url}

\newcommand{\N}{\mathbb{N}}

\newcommand{\R}{\mathbb{R}}
\newcommand{\Prob}{\mathbb{P}}
\newcommand{\hquad}{\hspace{0.5em}} 

\usepackage[utf8]{inputenc}
\usepackage[english]{babel}
\newtheorem{theorem}{Theorem}[]

\newtheorem{definition}[theorem]{Definition}
\newtheorem{lemma}[theorem]{Lemma}
\newtheorem{proposition}[theorem]{Proposition}
\newtheorem{remark}[]{Remark}

\begin{document}
\title{Non-Gaussian Chance-Constrained Trajectory Planning for Autonomous Vehicles under Agent Uncertainty}

\author{Allen Wang, Ashkan Jasour, and Brian C. Williams
\thanks{Manuscript received: February, 24, 2020; Revised June, 2, 2020; Accepted July, 2, 2020.} 
\thanks{This paper was recommended for publication by Editor Nancy Amato upon evaluation of the Associate Editor and Reviewers' comments.
This work was supported in part by the Boeing grant 6943358 and Masdar Institute grant 6938857. Allen Wang was supported in part by an NSF Graduate Research Fellowship. Embotech generously supplied academic FORCES Pro licenses.} 
\thanks{All authors are with the Computer Science and Artificial Intelligence Laboratory, Massachusetts Institute of Technology, Cambridge, MA
01239, USA \tt\footnotesize\{allenw, jasour, williams\}@mit.edu}%
\thanks{Digital Object Identifier (DOI): see top of this page.}
}

%

\markboth{IEEE Robotics and Automation Letters. Preprint Version. Accepted July, 2020}
{Wang \MakeLowercase{\textit{et al.}}: Non-Gaussian Chance-Constrained Trajectory Planning} 

\maketitle

\begin{abstract}
Agent behavior is arguably the greatest source of uncertainty in trajectory planning for autonomous vehicles. This problem has motivated significant amounts of work in the behavior prediction community on learning rich distributions of the future states and actions of agents. However, most current works on chance-constrained trajectory planning under agent or obstacle uncertainty either assume Gaussian uncertainty or linear constraints, which is limiting, or requires sampling, which can be computationally intractable to encode in an optimization problem. In this paper, we extend the state-of-the-art by presenting a methodology to upper-bound chance-constraints defined by polynomials and mixture models with potentially non-Gaussian components. Our method achieves its generality by using statistical moments of the distributions in concentration inequalities to upper-bound the probability of constraint violation. With this method, optimization-based trajectory planners can plan trajectories that are chance-constrained with respect to a wide range of distributions representing predictions of agent future positions. In experiments, we show that the resulting optimization problem can be solved with state-of-the-art nonlinear program solvers to plan trajectories fast enough for use online.
\end{abstract}


\section{INTRODUCTION}
\IEEEPARstart{I}{n} order for autonomous vehicles to drive safely on public roads, they need to plan trajectories that take into account predictions of future positions of other agents (e.g. human driven vehicles, pedestrians, cyclists). However, predictions are inherently uncertain, especially predictions of human behavior. This fact is motivating significant amounts of work in the behavior prediction community to develop methods that predict \textit{distributions} of future agent states and actions, often using a deep neural network (DNN). For example, \cite{lee2017desire} trains a conditional variational autoencoder to generate samples of possible future trajectories; the DNN in this case essentially becomes the distribution from which samples can be drawn. \cite{chai2019multipath, deo2018multi, hong2019rules} learn Gaussian mixture models (GMMs) for the agents' future positions to handle both uncertainty in high level decisions, which tends to be multi-modal, and uncertainty in execution, which tends to be continuous.

While work in behavior prediction can now generate rich distributions of future agent states and actions, most current works in chance-constrained trajectory planning only address the unimodal case and either make Gaussian assumptions or require sampling \cite{schwarting2017parallel,calafiore2006scenario,blackmore2010probabilistic, blackmore2009convex,luders2010chance,cannon2017chance}. We argue that handling non-Gaussian distributions is important because, for example, almost any distribution for agent action propagated through nonlinear dynamics models will result in non-Gaussian position distributions. Gaussians also have unbounded supports; this is unrealistic as the reachable set of agents is bounded by physical laws in reality. To handle non-Gaussian uncertainty, some prior works take a sampling-based approach. This often makes optimization computationally intractable because thousands of constraints need to be introduced to even enforce chance-constraints on the order of $10^{-2}$ for any practical problem \cite{calafiore2006scenario, cannon2017chance,norden2019efficient}. To address non-Gaussian uncertainty in obstacle or agent positions without sampling, recent works apply the Cantelli or Chebyshev inequality, but they are currently restricted to linear constraints \cite{summers2018distributionally,renganathan2020towards, mesbah2016stochastic,paulson2020stochastic}. For non-Gaussian uncertainty, sums-of-squares programming has been applied to the problem of trajectory tracking for nonlinear systems and risk assessment in the presence of non-convex obstacles, but current computational limitations restrict it to applications amenable to leveraging offline computation \cite{jasour2019risk,jasour2019sequential,wang2020RSS,jasour2018moment,steinhardt2012finite}.

\textit{Statement of Contributions:} In this paper, we present a general chance-constrained trajectory planning formulation for autonomous vehicles that can handle mixtures of non-Gaussian distributions of agent position and, unlike many prior works which make point mass assumptions, accounts for the sizes of the ego vehicle and agents. This is enabled by a general methodology we develop for enforcing non-Gaussian polynomial chance-constraints using concentration inequalities, extending the prior art which can only handle linear constraints. This methodology makes heavy use of symbolic algebra, and we develop and provide a Python package, \text{AlgebraicMoments}\footnote{The source code with examples can be found at \url{github.com/allen-adastra/algebraic_moments}.}, to implement it. Given a constraint defined as a polynomial in a random vector, AlgebraicMoments can generate a closed form expression that upper-bounds the probability of the event in terms of statistical moments of the random vector. It can even directly generate MATLAB or Python code to compute the risk bound, given the necessary inputs. Since this approach only depends on statistical moments of the distributions, it can apply to a wide range of prediction distributions including non-Gaussian mixture models of future agent positions. In numerical experiments, we show how our formulation, when solved with advanced interior-point methods, can be used to plan trajectories with horizons of $5$ seconds fast enough for use online. While this paper focuses on the problem of trajectory planning for autonomous vehicles under agent uncertainty, the general methodology should be applicable to motion and trajectory planning problems with constraint uncertainty in many other application domains.
\section{Representation of Agent Predictions}
\subsection{Assumptions}
We assume a behavior prediction system provides the distribution of future positions for an agent over a $T$ step horizon, $\mathbf{g}_{1:T}$, in a fixed frame. The distributions can be either unimodal or a mixture of non-Gaussian random vectors. The distributions $\mathbf{g}_t = [g_{x_t}, g_{y_t}]^T$ are assumed to be either independent across time in the unimodal case or independent across time conditioned on the discrete mode in the mixture model case. This is a common assumption used in state-of-the-art behavior prediction systems; thus, it does not significantly restrict our method's range of applicability \cite{chai2019multipath,deo2018multi,rhinehart2018r2p2}. We also make the additional assumption that the predicted distributions do not change w.r.t. changes in the ego vehicle trajectory, as accounting for the change requires having the behavior prediction system in the planning loop. In practice, it may be more effective to alternate between the planner and prediction systems by using planned trajectories in the prediction system to generate a new distribution, but we do not explore this interaction in this paper.

\subsection{Computing Moments of Distributions}
In some cases, the statistical moments of distributions are known in closed form. In other cases, statistical moments of a distribution can be rapidly computed by applying automatic or numerical differentiation to its characteristic function (CF). This is a very general approach to computing moments of distributions as CFs always exist, and, from a more practical standpoint, there are extensive tables of CFs for common distributions \cite{witkovsky2017brief}. Letting $X$ denote a random variable and $\Phi_X(t)$ denote its characteristic function, the $n_{th}$ moment of $X$ can be computed by:
\begin{align}
    \mathbb{E}[X^n] = i^{-n}\left[\frac{d^n}{dt^n}\Phi_X(t)\right]_{t=0}
\end{align}
Similarly, moments of a random vector, $\mathbf{w}$, can be computed via partial differentiation of its joint characteristic function $\Phi_\mathbf{w}(t)$, although we note catalogues of CFs are less extensive for multivariate distributions. Alternatively, moments may also be estimated with Monte Carlo type methods; this is useful for approaches where the DNN is a distribution from which sampled trajectories are drawn.

\subsection{Statistics of Mixture Models}\label{subsec:mm_statistics}
A mixture model is a rich way of expressing multi-modal uncertainty by combining multiple continuous distributions \cite{fruhwirth2006finite}. In this work, we work with random vectors and define random vector mixture models as:
\begin{definition}\normalfont
An $n$ component random vector mixture model $\mathbf{w}$ is a random vector with components $\mathbf{w}_i$ and mixture weights $w_i$ for $i=1,...,n$ s.t. $\sum_{i=1}^n w_i=1$. Its pdf $f_{\mathbf{w}}$ is related to those of its components $f_{\mathbf{w}_i}$ by:
\begin{align}
    f_\mathbf{w}(\cdot) = \sum_{i=1}^nw_if_{\mathbf{w}_i}(\cdot)
\end{align}
\end{definition}
A useful property of mixture models is its statistics can be computed in terms of statistics of its components \cite{fruhwirth2006finite}. Proposition \ref{prop:decompose_exp_g} states this fact in the general case and only makes the mild assumption that $g(\cdot)$ is measurable, allowing it to apply to most functions used in practice. As a simple corollary, by letting $g(\mathbf{w})$ be some moment of $\mathbf{w}$, moments of $\mathbf{w}$ can be expressed as the weighted sum of the moments of its components.
\begin{proposition}\label{prop:decompose_exp_g}\normalfont
For any $n$ component random vector mixture model $\mathbf{w}$ with components $\mathbf{w}_i$ with mixture weights $w_i$ and any measurable function $g$, we have that $\mathbb{E}[g(\mathbf{w})] = \sum_{i=1}^n w_i\mathbb{E}[g(\mathbf{w}_i)]$.
\end{proposition}
\begin{proof}
By the law of the unconscious statistician and the definition of the mixture model pdf:
\begin{align}
    \mathbb{E}[g(\mathbf{w})] = \int g(\mathbf{x})\sum_{i=1}^nw_if_{\mathbf{w}_i}(\mathbf{x}) d\mathbf{x}
\end{align}
Applying the linearity of expectation we then have:
\begin{align}
    \int g(\mathbf{x})\sum_{i=1}^nw_if_{\mathbf{w}_i}(\mathbf{x}) d\mathbf{x} &= \sum_{i=1}^n w_i\int g(\mathbf{x})f_{\mathbf{w}_i}(\mathbf{x}) d\mathbf{x}\\
    &= \sum_{i=1}^nw_i\mathbb{E}[g(\mathbf{w}_i)]
\end{align}
\end{proof}
\section{Problem Formulation}
\subsection{Definition of Risk}
We define risk in a way that accounts for the size of the agent and ego vehicle. The general idea is to fit circles of radius $r$ to an agent and constrain the probability that the centers of the circles are inside an appropriately scaled ``collision ellipsoid" around the vehicle. Figure \ref{fig:collision_ellipses} illustrates an example; note that if the centers of the circles are not in the ellipsoid, the vehicles are not in collision.
\begin{figure}[!h]
    \centering
    \includegraphics[width=\linewidth]{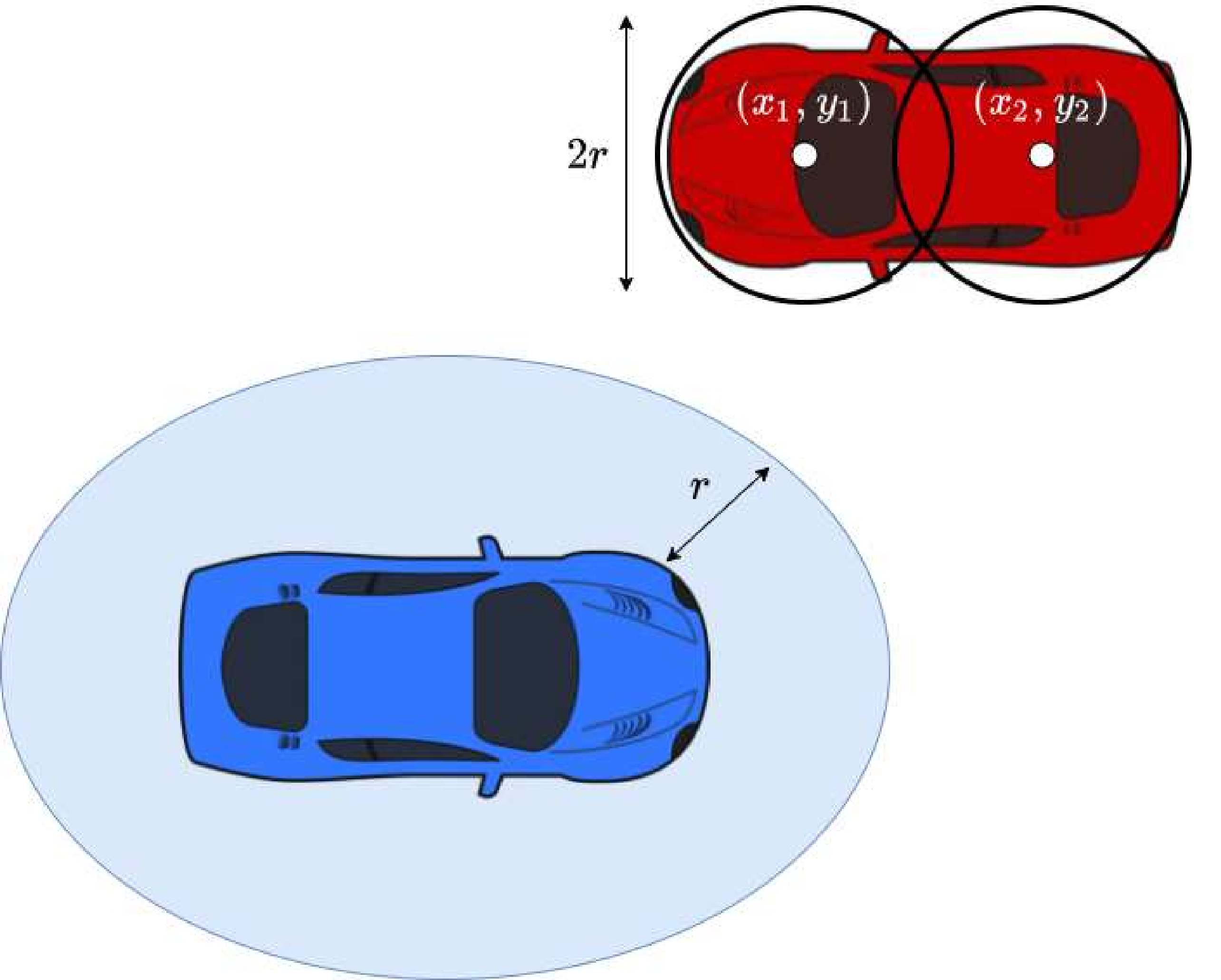}
    \caption{An example showing a collision ellipsoid around  the ego vehicle and corresponding circles around the agent. Note that if the points $(x_i, y_i)$ are not in the ellipsoid, then the vehicles are not in collision.}
    \label{fig:collision_ellipses}
\end{figure}
\begin{figure}[!h]
    \centering
    \includegraphics[width=\linewidth]{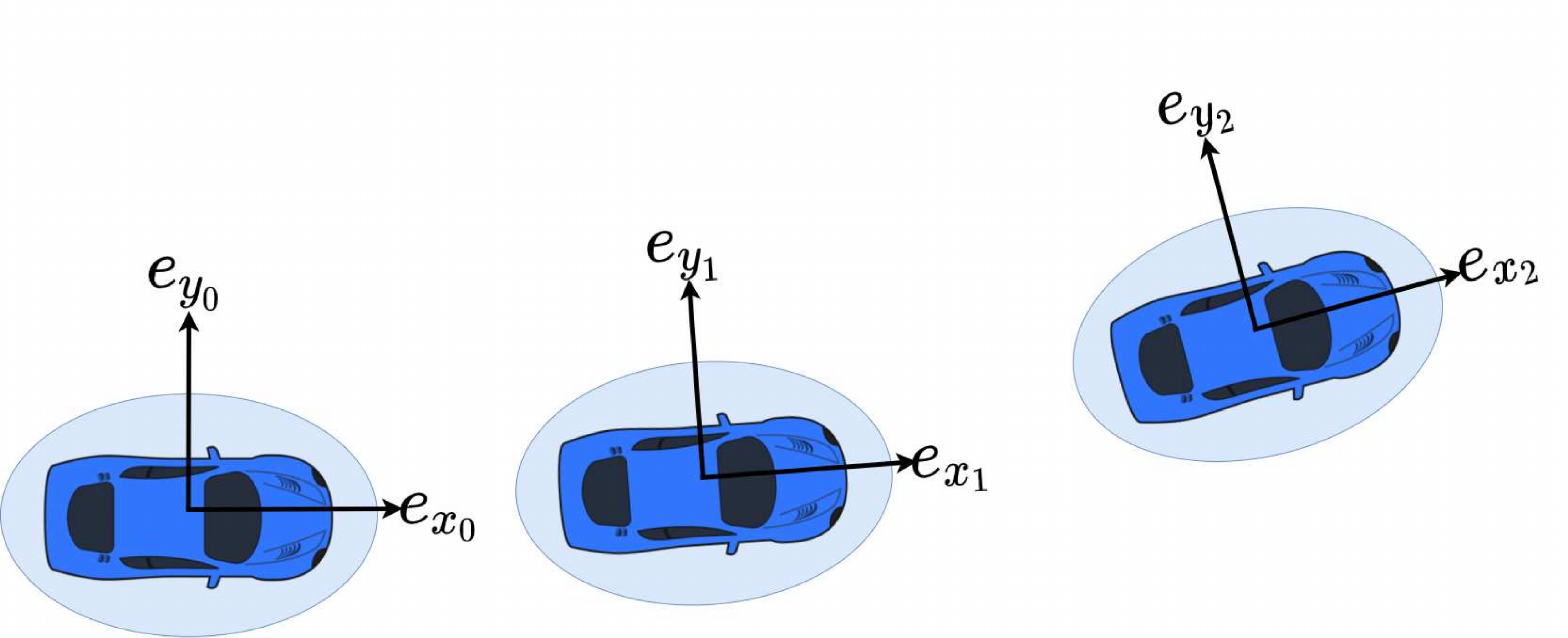}
    \caption{A planned trajectory for the ego vehicle along with the ``planned body frames" and collision ellipsoids drawn around the vehicle.}
    \label{fig:ego_planned_frames}
\end{figure}
We essentially treat each circle as a separate agent. We will define risk in the ``planned body frames", which are depicted by figure \ref{fig:ego_planned_frames}. While the prediction is given in some fixed frame, we eventually show in section \ref{subsec:app_rb} how moments in the planned body frames can be expressed as a function of the fixed frame moments and the planned ego vehicle pose. In the body frame, the ellipsoid is the set:
\begin{align}
    \{\mathbf{x} : \mathbf{x}^TQ\mathbf{x}\leq 1\}
\end{align}
Where $Q$ is a constant $2\times 2$ positive definite matrix. The following notation will be used for quadratic forms as it better reflects the polynomial nature of quadratic forms, and many results in this paper will be for polynomials in general:
\begin{align}
    Q(\mathbf{x}):=\mathbf{x}^TQ\mathbf{x}
\end{align}
We denote the distribution of the $i_{th}$ agent in the planned frame at time $t$ as $\mathbf{a}_{t,i}$; recall this distribution is ultimately a function of the global frame prediction $\mathbf{g}_{t,i}$ and the planned ego vehicle pose. Risk along the horizon is thus defined as:
\begin{align}
    \mathcal{R}:=\Prob\left(\bigcup_{i=1}^{n_a}\bigcup_{t=1}^T\{Q(\mathbf{a}_t)\leq 1\}\right)\label{eq:union_events}
\end{align}
However, evaluating (\ref{eq:union_events}) without excessive conservatism is a non-trivial problem itself meriting a separate treatment. For this paper, we simply apply Boole's Inequality which states:
\begin{align}
    \mathcal{R}\leq \sum_{i=1}^{n_a}\sum_{t=1}^T\Prob(Q(\mathbf{a}_{t,i})\leq1)
\end{align}
To upper bound the total risk $\mathcal{R}$ with some $\Delta$, we define a risk-allocation $\epsilon_{t,i}$ s.t.:
\begin{align}
    \sum_{i=1}^{n_a}\sum_{t=1}^T\epsilon_{t,i}\leq\Delta
\end{align}
And then upper-bound the marginal probabilities as such:
\begin{align}
    \Prob(Q(\mathbf{a}_{t,i})\leq 1)\leq\epsilon_{t,i}
\end{align}
It is possible to encode the risk allocation as decision variables in an optimization problem, but doing so can often be computationally intractable. In this paper, we solve the trajectory planning problem with a fixed risk allocation; other works have presented approaches for ``outer loops" that optimize the risk allocation \cite{ono2008iterative,paulson2020stochastic,ma2012fast}. In addition, throughout the rest of this paper, we present mathematical statements for the single agent case for the sake of notational simplicity; it is straight-forward to simply repeat the chance-constraint in the multi-agent case.
\subsection{The Trajectory Planning Problem}
In this paper, we derive results for the general cc-trajectory planning problem defined below.
\begin{subequations}\label{eq:cc_optimal_control}
\begin{align}
    \min_{\mathbf{x}_{1:T}, \mathbf{u}_{0:T}}\quad & c(\mathbf{x}_{1:T}, \mathbf{u}_{1:T})\label{eq:cc_mpc_cost}\\
    & \mathbf{x}_{t+1} = f(\mathbf{x}_t, \mathbf{u}_t), \hquad t\in[T-1]\\
    & \Prob(Q(\mathbf{a}_t)\leq 1)\leq \epsilon_t, \hquad t\in[T]\label{eq:prob_constraint}\\
    & \mathbf{u}_{min}\leq \mathbf{u}_t\leq \mathbf{u}_{max}, \hquad t\in[T]\\
    & \mathbf{x}_{min}\leq \mathbf{x}_t\leq \mathbf{x}_{max}, \hquad t\in[T]
\end{align}
\end{subequations}
Where $\mathbf{x}_t$ is the state vector, $\mathbf{u}_t$ is the control vector, $c$ is some cost function, $f$ is a discrete time system modeling ego vehicle dynamics, $\mathbf{u}_{min}$ and $\mathbf{u}_{max}$ are control limits and $\mathbf{x}_{min}$ and $\mathbf{x}_{max}$ are state limits. We assume that the ego vehicle dynamics are deterministic as modern feedback control systems for autonomous vehicles are effective at tracking trajectories with positional error on the order of ten centimeters \cite{hoffmann2007autonomous, wang2008autonomous}. Thus, the only difference between this problem and standard deterministic trajectory planning formulations is the chance constraint (\ref{eq:prob_constraint}) which ensures that the probability of the vehicle colliding with an agent is no more than $\epsilon$ at each time step. Section \ref{sec:chance_constraint} presents our approach to enforcing a risk-bound on this chance constraint.

\section{Bounding Non-Gaussian Polynomial Chance-Constraints}\label{sec:chance_constraint}
In this section, we present a general methodology for bounding polynomial non-Gaussian chance-constraints; that is, for a random vector $\mathbf{w}$ and polynomial $p$, we present a method to establish a bound on the probability:
\begin{align}
    \Prob(p(\mathbf{w})\leq 0)\label{eq:general_cc}
\end{align}
Recall that the chance-constraint in our trajectory planning problem (\ref{eq:prob_constraint}) takes this form as quadratic forms are polynomials. By viewing $p(\mathbf{w})$ as a random variable, (\ref{eq:general_cc}) is simply the cumulative distribution function (cdf) of $p(\mathbf{w})$ evaluated at $0$. Unfortunately, even in the relatively simple case where $\mathbf{w}$ is a multivariate Gaussian and $p$ is a quadratic, $p(\mathbf{w})$ does not have a closed form cdf \cite{duchesne2010computing}. In general, it can very challenging to characterize the cdfs of distributions that arise from nonlinear transformations, so our approach is to leverage one-tailed concentration inequalities which bound (\ref{eq:general_cc}) using the mean and variance of $p(\mathbf{w})$. Subsection \ref{subsec:moments_poly} begins by showing how moments of the random variable $p(\mathbf{w})$ can be expressed in closed form in terms of moments of $\mathbf{w}$. This provides us with a way to compute the mean and variance of $p(\mathbf{w})$ given the distribution of $\mathbf{w}$. Subsection \ref{subsec:concentration_ineq} then introduces several concentration inequalities that can be used to bound (\ref{eq:general_cc}) using the mean and variance of $p(\mathbf{w})$. We can directly apply these approaches to mixture models by computing the moments of the mixture models with the approach from subsection \ref{subsec:mm_statistics}. However, we show in subsection \ref{subsec:mixture_model} that tighter bounds can be achieved by instead bounding the components of the mixture model. Finally, subsection \ref{subsec:app_rb} applies these techniques to bound the chance-constraint (\ref{eq:prob_constraint}).

\subsection{Moments of Polynomials in Random Vectors}\label{subsec:moments_poly}
An important property of $p$ being a polynomial is that the $n_{th}$ moment of $p(\mathbf{w})$ can be computed as the weighted sum of moments of $\mathbf{w}$. This is true because $p(\mathbf{w})^n$ is, itself, a polynomial to which the linearity of expectation can be applied. To see this, consider the following simple example where $p(\mathbf{w}) = w_1^2 + w_2^2$:
\begin{subequations}
\begin{align}
    \mathbb{E}[p(\mathbf{w})^2] &= \mathbb{E}[w_1^4 + 2w_1^2w_2^2 + w_2^4]\\
    &= \mathbb{E}[w_1^4] + 2\mathbb{E}[w_1^2w_2^2] + \mathbb{E}[w_2^4]
\end{align}
\end{subequations}
To state the general case, we adopt multi-index notation which allows us to much more succinctly express moments of random vectors. For example, $\mathbb{E}[w_1^2w_2^2]$ can be represented with the vector $\mathbf{w}$ and a multi-index $\alpha=(2,2,0,...,0)$. Letting $\mathbf{w}$ be an $n$ dimensional random vector, we can express any moment of $\mathbf{w}$ with a multi-index $\alpha\in\mathbb{N}^n$ as such where $\alpha_i$ is the $i_{th}$ element of $\alpha$:
\begin{align}
    \mathbb{E}[\mathbf{w}^\alpha] := \mathbb{E}\left[\prod_{i=1}^nw_i^{\alpha_i}\right]
\end{align}
Proposition \ref{prop:moments_poly} uses multi-index notation to express the idea that any moment of $p(\mathbf{w})$ can be expressed in terms of moments of $\mathbf{w}$. AlgebraicMoments can be used to easily derive expressions of the form (\ref{eq:moment_expression}), and it even leverages independence in the random vector $\mathbf{w}$ to further decompose terms of the form $\mathbb{E}[\mathbf{w}^\alpha]$.
\begin{proposition}\label{prop:moments_poly}\normalfont
For a $n$ dimensional random vector $\mathbf{w}$, a polynomial $p$ and $m\in\N$, there exists a set of multi-indices $\mathcal{A}\subset\mathbb{N}^n$ and coefficients $C_\mathcal{A} = \{c_\alpha\in\R : \alpha\in\mathcal{A}\}$ s.t.:
\begin{align}
    \mathbb{E}[p(\mathbf{w})^m] = \sum_{\alpha\in\mathcal{A}}c_\alpha\mathbb{E}[\mathbf{w}^\alpha]\label{eq:moment_expression}
\end{align}
\end{proposition}
\begin{proof}
Since $p$ is a polynomial, $p(\mathbf{w})^m$ is also a polynomial in $\mathbf{w}$ since the ring of polynomials is closed under multiplication. Thus, we have the existence of $\mathcal{A}$ and $C_\mathcal{A}$ s.t. $p(\mathbf{w})^m = \sum_{\alpha\in\mathcal{A}}c_\alpha\mathbf{w}^\alpha$. Applying the expectation operator to both sides and the linearity of expectation, we arrive at the result.
\end{proof}

\subsection{Bounding Risk with Concentration Inequalities}\label{subsec:concentration_ineq}
The prior subsection shows how moments of $p(\mathbf{w})$ can be expressed in terms of moments of $\mathbf{w}$; in this section we show how the mean and variance of $p(\mathbf{w})$, $\mu_{p(\mathbf{w})}$ and $\sigma_{p(\mathbf{w})}^2$, can be used to bound risk using concentration inequalities. We start with Cantelli's inequality \cite{haimes2012research}, also known as the one-tailed Chebyshev Inequality, which bounds the probability of constraint violation as such:
\begin{equation}\label{eq:cheb}
    \Prob(p(\mathbf{w})\leq 0)\begin{cases}\leq \frac{\sigma^2_{p(\mathbf{w})}}{\sigma^2_{p(\mathbf{w})} + \mu_{p(\mathbf{w})}^2} & \mu_{p(\mathbf{w})} \geq 0\\
    \geq 1 - \frac{\sigma^2_{p(\mathbf{w})}}{\sigma^2_{p(\mathbf{w})} + \mu_{p(\mathbf{w})}^2} & \mu_{p(\mathbf{w})} < 0
    \end{cases}
\end{equation}
However, for many applications, Cantelli's inequality can be excessively conservative; in fact, it is often sharp only for discrete distributions. By making additional mild assumptions, we can arrive at tighter bounds with the Vysochanskij-Petunin (VP) and Gauss inequalities, which are very similar \cite{padulo2011worst,popescu2005semidefinite,vysochanskij1980justification}. Our proposed strategy is to adopt the tightest inequality for which $p(\mathbf{w})$ meets the given assumptions; the table below summarizes the assumptions for each inequality.
\begin{table}[!h]
\centering
\caption{Assumptions required for concentration inequalities.}
\begin{tabular}{|l|l|}
\hline
\textbf{Inequality} & \textbf{Assumptions}                     \\ \hline
Cantelli   & $p(\mathbf{w})$ has finite mean + variance          \\ \hline
Vysochanskij-Petunin (VP)         & Cantelli assumptions + unimodal \\ \hline
Gauss      & VP assumptions + symmetric pdf      \\ \hline
\end{tabular}
\label{table:concentration_inequalties_constraints}
\end{table}
Since the inequalities are very similar, we simplify future notation by defining the $\text{Conc}[\cdot]$ operator to denote the tightest appropriate concentration inequality and $\text{Conc}^*[\cdot]$ to denote the corresponding necessary condition for the inequality to hold. In practice, the practitioner would have to select the correct inequality to use on a case-by-case basis.
\begin{align}\label{eq:con_constar}
    \text{Conc}[p(\mathbf{w})] &= \begin{cases}
        \frac{\sigma_{p(\mathbf{w})}^2}{\sigma_{p(\mathbf{w})}^2 + \mu_{p(\mathbf{w})}^2}  \quad &\scriptsize\text{Cantelli Holds}\\
        \frac{4}{9}\frac{\sigma_{p(\mathbf{w})}^2}{\sigma_{p(\mathbf{w})}^2 + \mu_{p(\mathbf{w})}^2} &\scriptsize\text{VP Holds}\\
        \frac{2}{9}\frac{\sigma_{p(\mathbf{w})}^2}{\mu_{p(\mathbf{w})}^2} &\scriptsize\text{Gauss Holds}
    \end{cases}
\end{align}
\begin{align}
        \text{Conc}^*[p(\mathbf{w})] &= \begin{cases}
    -\mu_{p(\mathbf{w})} \quad & \scriptsize\text{Cantelli Holds}\\
    -\mu_{p(\mathbf{w})} + \sqrt{\frac{5}{3}}\sigma_{p(\mathbf{w})} & \scriptsize\text{VP Holds}\\
    -\mu_{p(\mathbf{w})} + \frac{2}{3}\sigma_{p(\mathbf{w})} &\scriptsize\text{Gauss Holds}
    \end{cases}
\end{align}
Thus, an $\epsilon$ chance-constrain can be bounded in an optimization problem by enforcing the constraints:
\begin{align}
    \text{Conc}[p(\mathbf{w})]\leq\epsilon\\
    \text{Conc}^*[p(\mathbf{w})]\leq 0
\end{align}
Note that the $\text{Conc}^*$ conditions are not particularly restrictive, as they are only violated when risk is relatively high. In fact, the $\text{Conc}^*$ condition only breaks when the upper-bound is at least $1, 1/6$, or $1/2$ for the Cantelli, VP, and Gauss cases respectively, which is well above values usually specified for chance-constraints in practice.

\subsection{Tighter Bounds for Mixture Models}\label{subsec:mixture_model}
In the case that $\mathbf{w}$ is a mixture model, we can simply compute the moments of $p(\mathbf{w})$ using the result of Proposition \ref{prop:decompose_exp_g} in terms of moments of the components $p(\mathbf{w}_i)$ and treat the mixture model as any other random variable. However, Proposition \ref{prop:decompose_exp_g} also seems to suggest that we can instead bound the component probabilities $\Prob(p(\mathbf{w}_i)\leq 0)$ and bound the overall risk with the weighted sum of the component bounds. Intuition suggests this approach may be better because it involves applying concentration inequalities at the most detailed level possible. Since concentration inequalities are, in a sense, blanket statements about distributions with a given set of moments, it makes sense to apply them at the most detailed level possible. In fact, we show with Theorem \ref{thm:cheb_mms} that applying concentration inequalities to the mixture components individually will almost certainly produce a less conservative risk bound.
\begin{theorem}\label{thm:cheb_mms}\normalfont
For any random vector mixture model $\mathbf{w}$ with $n$ components $\mathbf{w}_i$ and weights $w_i$ and any measurable function $g$, if $\text{Conc}^*[g(\mathbf{w}_i)]\leq 0, \forall i\in[n]$, then:
\begin{subequations}
\begin{align}
    \Prob(g(\mathbf{w})\leq 0)&\leq \sum_{i=1}^n w_i\text{Conc}(g(\mathbf{w}_i))\label{eq:decompose_cheb}\\
    &\leq \text{Conc}(g(\mathbf{w}))\label{eq:mm_thm_leq}
\end{align}
\end{subequations}
For Cantelli and VP, almost surely, we have:
\begin{align}
    \sum_{i=1}^n w_i\text{Conc}(g(\mathbf{w}_i)) < \text{Conc}(g(\mathbf{w}))\label{eq:mm_conservatism}
\end{align}
\end{theorem}
\begin{proof}
See appendix.
\end{proof}
\begin{remark}\normalfont
``Almost surely" in Theorem \ref{thm:cheb_mms} means that if the first and second moments of each component are randomly chosen from $\R^2$ according to any distribution supported on a subset of $\R^2$, then the result holds with probability one.
\end{remark}
We arrive at the ``almost surely" result by establishing a geometric sufficient condition in the space of first and second moments. We show that each component of a mixture model generates a line in the first-second moment space. If the first-second moment point of each component does not lie on the line generated by any other component, the result holds. Since the union of a finite number of lines has Lebesgue measure zero with respect to $\R^2$, the probability of randomly choosing first and second moments that violate the condition is zero. Even if the condition is violated, an arbitrarily small perturbation of the moments would satisfy the condition. Figure \ref{fig:random_first_second_moments} illustrates an example that satisfies this condition.
\begin{figure}[!h]
    \centering
    \includegraphics[width=0.9\linewidth]{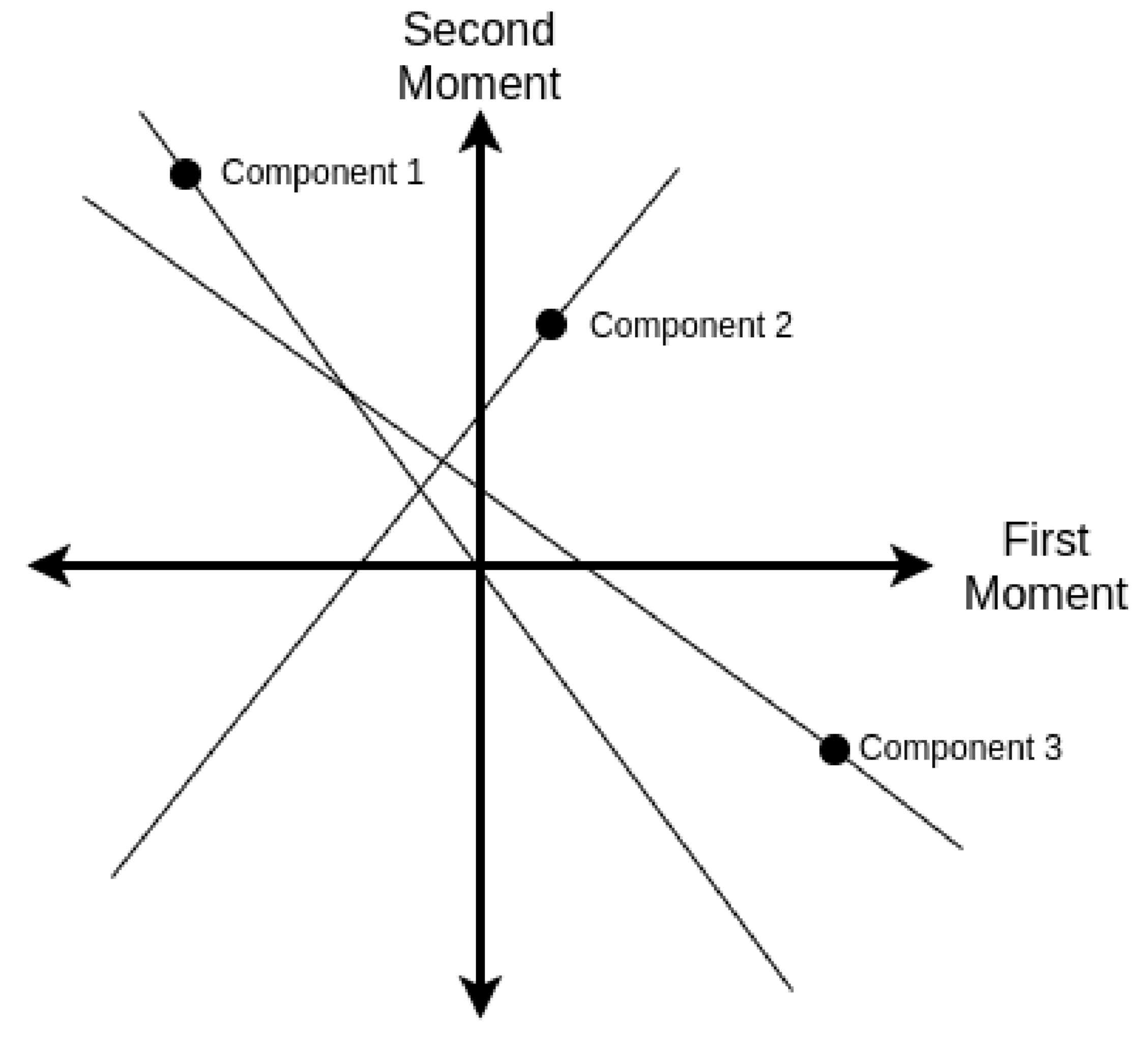}
    \caption{The first and second moments of three components of a mixture model are plotted along with the lines they generate. Note how every component lies on the line it generates, but does not lie on the lines generated by any other component. Thus, for this example, the ``almost surely" result of Theorem \ref{thm:cheb_mms} applies.}
    \label{fig:random_first_second_moments}
\end{figure}

\subsection{Bounding Risk for the Trajectory Planning Problem}\label{subsec:app_rb}
We now return to the problem of establishing a bound on:
\begin{align}
    \Prob(Q(\mathbf{a}_t)\leq 1)
\end{align}
Letting $\mathbf{a}_t^{(i)}$ denote the $i_{th}$ component of $\mathbf{a}_t$, this essentially consists of applying the methods of the prior subsections to the polynomial:
\begin{align}
    Q(\mathbf{a}_t^{(i)})-1
\end{align}
However, there is one small additional complication: $\mathbf{a}_t^{(i)}$ is in the body frame, making it a function of the planned ego vehicle pose and the global frame distribution $\mathbf{g}_t^{(i)}$. Letting $\mathbf{y}_t = [x_t, y_t]$ denote the ego vehicle position, $\theta_t$ denote the ego vehicle heading, and $R(\cdot)$ denote the 2D rotation matrix, they are related by:
\begin{align}
    \mathbf{a}_t^{(i)} = R(\theta_t)^T(\mathbf{g}_t^{(i)}-\mathbf{y}_t)
\end{align}
The idea here is to view $\mathbf{a}_t^{(i)}$ as a polynomial in $\mathbf{g}_t^{(i)}$ and to view the elements of both $R(\theta_t)^T$ and $\mathbf{y}_t$ as elements of the coefficients. From this perspective, we apply AlgebraicMoments to express moments of $\mathbf{a}_t^{(i)}$ in terms of moments of $\mathbf{g}_t^{(i)}$. The moments of $Q(\mathbf{a}_t^{(i)})-1$ can, in turn, be expressed in terms of the moments of $\mathbf{a}_t^{(i)}$. With the first and second moments of $Q(\mathbf{a}_t^{(i)})-1$ expressed in closed form, we can enforce the following constraints in the optimization problem where $n$ is the number of mixture components:
\begin{align}
    &\sum_{i=1}^n\text{Conc}[Q(\mathbf{a}_t^{(i)})-1]\leq\epsilon_t,\hquad t\in[T]\\
    &\text{Conc}^*[Q(\mathbf{a}_t^{(i)})-1]\leq 0,\hquad i\in[n]\hquad t\in[T]
\end{align}
AlgebraicMoments provides the functionality to directly generate code for the above expressions given a specification.

\section{Application to Trajectory Planning with Model Predictive Contouring Control}
In the context of autonomous driving, an approximate coarse-grained path is almost always available to trajectory planners as most autonomous driving systems have maps of lane geometries \textit{a priori} and also have routers and higher level planners that provide a discrete plan (e.g: a sequence of waypoints along the road). This availability of a ``reference path" makes model predictive contouring control (MPCC) a useful approach. MPCC is a methodology for expressing an approximation of the minimum distance from a point to a third order polynomial in closed form. This allows for deviations from the reference path to be included in the cost function of an optimization problem. By jointly applying MPCC and our chance-constraint formulation, we arrive at a trajectory planner that can find trajectories with low deviation from the reference path while satisfying some desired level of safety, allowing for the generation of rich qualitative behavior. In the following subsections, we present a brief overview of the MPCC formulation; the interested reader is referred to \cite{schwarting2017parallel, lam2012model} for additional details.

\subsection{Reference Path}\label{subsec:ref_path}
Following the standard contouring control formulation, the reference path is represented as third order polynomials in an arc-length parameter $s\in[0, L]$ where $L$ is the length:
\begin{align}
    \begin{bmatrix}
    x_{ref}(s)\\
    y_{ref}(s)
    \end{bmatrix} = 
    \begin{bmatrix}
    c_{x0} + c_{x1}s + c_{x2}s^2 + c_{x3}s^3\\
    c_{y0} + c_{y1}s + c_{y2}s^2 + c_{y3}s^3
    \end{bmatrix}
\end{align}
Where $c_{xi}$ and $c_{yi}$ for $i\in[3]$ are the polynomial coefficients. Parameterizing a third order polynomial path with arc-length is a difficult problem itself without exact solutions, but approximation methods are well-studied \cite{peterson2006arc, floater2005arc}. For our experiments, we applied a simple approximation by initially generating the polynomials with $s\in[0, 1]$ and then the lengths of the polynomials were computed by numerical integration. The coefficients are then scaled s.t. $s\in[0, L]$. The heading at each point on the reference path, denoted $\Theta(s)$, can be expressed as such:
\begin{align}
    \Theta(s) = \arctan\left(\frac{\partial y_{ref}}{\partial x_{ref}}(s)\right)
\end{align}
\subsection{Contouring Deviation and Lag Error}\label{subsec:contour_lag}
Ideally, the Euclidean distance from the ego vehicle to the nearest point on the reference path, which we will refer to as \textit{contouring deviation}\footnote{In the literature, this is usually known as \textit{contouring error}, but we call it contouring deviation as deviation from the reference path to satisfy chance-constraints is not necessarily undesirable.}, would be used as the measure of deviation from the reference path, but doing so requires a minimization over the path parameter $s$ that is computationally intractable to perform in an optimization routine. The standard solution is to approximate the contouring deviation by using the distance the vehicle has travelled, denoted $\Delta$, as an approximation. This only requires adding an additional integrator variable to the dynamics model as the time derivative of $\Delta$ is the vehicles speed. Letting $\bar{x}_t = x_t - x_{ref}(\Delta_t)$ and $\bar{y}_t = y_t - y_{ref}(\Delta_t)$, contouring deviation can be approximated with:
\begin{align}
    D_t = \sin(\Theta(\Delta_t))\bar{x}_t - \cos(\Theta(\Delta_t))\bar{y}_t
\end{align}
It is also important to penalize error between $\Delta$ and the true parameter corresponding to the closest point on the path to the vehicle. This quantity is known as the \textit{lag error} and can be approximated by:
\begin{align}
    L_t = -\cos(\Theta(\Delta_t))\bar{x}_t - \sin(\Theta(\Delta_t))\bar{y}_t
\end{align}

\subsection{Ego Vehicle Model}
For driving in nominal conditions, the kinematic bicycle model is known to provide a high level of fidelity while requiring less computational cost than a dynamics model making it well suited for trajectory planning \cite{kong2015kinematic}. The state of the vehicle is defined as $\mathbf{x} = [x, y, \theta, v, \delta, \Delta]^T$ where $x, y$ denotes the vehicle's position and $\theta$ denotes the heading in the global coordinates. $v$ denotes speed, $\delta$ denotes the front steering angle, and $\Delta$ denotes the distance traveled. The control inputs are $\mathbf{u} = [u_a, u_\delta]^T$ where $u_a$ is acceleration and $u_\delta$ is the rate of change of the steering angle. The relevant physical parameters of the vehicle in this model are the distances from the center of gravity to the front and rear axles; we denote them $l_f$ and $l_r$ respectively. The continuous time model is thus:
\begin{align}
    \dot{\mathbf{x}} = \begin{bmatrix}
    v\cos(\theta + \beta)\\
    v\sin(\theta + \beta)\\
    \frac{v}{l_r}\sin(\beta)\\
u_a\\
    u_\delta\\
    v
    \end{bmatrix} \quad     \beta = \arctan\left(\frac{l_r}{l_f + l_r}\tan(\delta)\right)
\end{align}

\subsection{Optimization Statement}
In the cost function, we penalize contouring deviation, lag error, control effort, and deviation from a reference speed (e.g: the roads speed limit):
\begin{align}
c(\mathbf{x}_t,\mathbf{u}_t) = c_{D}D_t^2 + c_L L_t^2 + \mathbf{u}_t^TR\mathbf{u}_t + c_v(v_t - v^*)^2
\end{align}
Where $c_D, c_L, c_v\in\R$ and $R\in\mathcal{S}_{++}^2$ are cost function parameters and $v^*$ is the reference speed. Letting $f_{RK4}$ denote the RK4 approximation of the bicycle model, the full problem we solve is:
\begin{subequations}
\begin{align}
    \min &\sum_{t=0}^Tc(\mathbf{x}_t,\mathbf{u}_t)\\
    &\mathbf{x}_{t+1} = f_{RK4}(\mathbf{x}_t,\mathbf{u}_t),\hquad t\in[T-1]\\
    &\sum_{i=1}^n\text{Conc}[Q(\mathbf{a}_t^{(i)})-1]\leq\epsilon_t,\hquad t\in[T]\label{eq:rk4_bicycle}\\
    &\text{Conc}^*[Q(\mathbf{a}_t^{(i)})-1]\leq 0,\hquad i\in[n]\hquad t\in[T]\\
    &\mathbf{x}_{min}\leq\mathbf{x}_t\leq\mathbf{x}_{max},\hquad t\in[T]\\
    &\mathbf{u}_{min}\leq\mathbf{u}_t\leq\mathbf{u}_{max},\hquad t\in[T-1]
\end{align}
\end{subequations}
\section{Numerical Experiments}\label{sec:numerical_experiments}
\begin{figure}
    \centering
    \includegraphics[width=\linewidth]{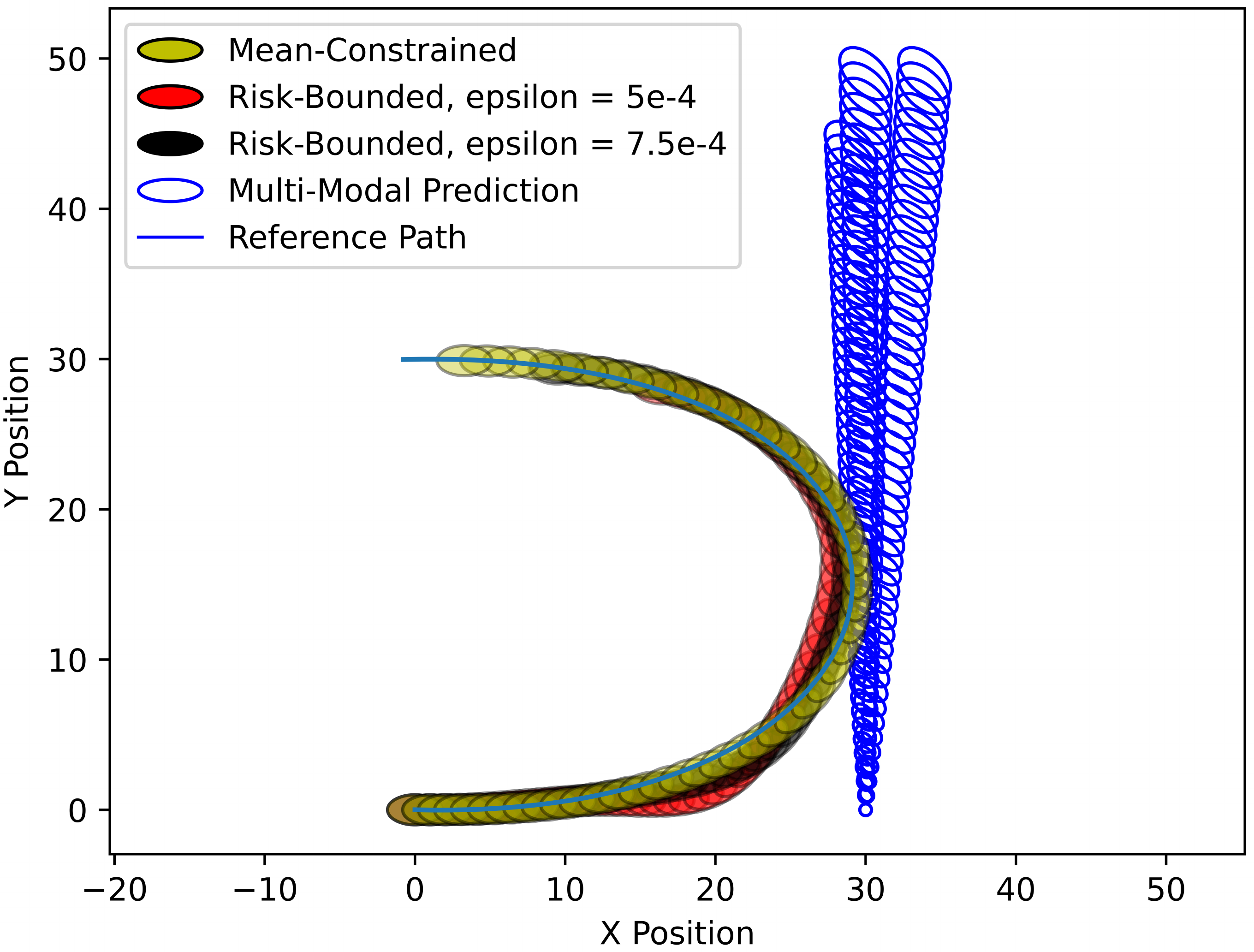}
    \caption{Comparison between trajectories planned with a deterministic mean-constrained formulation and the risk-bounded formulation with two different chance-constraints. Both risk-bounded trajectories exhibit more cautious behavior when performing the U-turn, by taking a wider and slower turn to satisfy chance-constraints with the lower chance-constraint producing a noticeably more conservative result. The prediction is represented with two standard deviation confidence ellipses.}
    \label{fig:deflect_off_reference_path}
\end{figure}
For numerical experiments, we used optimizers generated with FORCES Pro, a software package that generates high performance interior-point solvers for optimal control problems \cite{domahidi2014forces, zanelli2020forces}. A third order polynomial path emulating a U-turn was generated to serve as the reference path. Cost parameters are included in the appendix. Example GMM predictions were created by linearly interpolating the mean vectors and covariance matrices of manually specified end-point GMMs. To demonstrate the advantages of using non-Gaussian distributions with our approach, predictions consisting of mixtures of truncated multivariate Gaussians, which we will refer to as truncated GMMs, were also created by performing a box truncation of the GMM predictions at two standard deviations. We believe truncated GMMs are effective models for agents as reachability analysis can be applied in practice to determine bounds on the positions of agents. $50$ time steps were used with intervals of $0.1$ seconds for a planning horizon of $5$ seconds. Initial guesses for the optimizer were produced by simulating the system (\ref{eq:rk4_bicycle}) with constant control inputs.

Figure \ref{fig:deflect_off_reference_path} shows two trajectories planned with two different formulations: one deterministic and the other risk-bounded. The deterministic formulation uses the means of the agent predictions to enforce a deterministic constraint on collision. The risk-bounded formulation set $\epsilon=0.0005$ and $\epsilon=0.00075$ and utilized the VP inequality. Note how the risk-bounded formulation plans qualitatively more cautious trajectories by taking  wider turns and slower speeds. 

As mentioned earlier, the unbounded support of multivariate Gaussians limits their realism as models for agent position as the forward reachable set of agents is bounded by the laws of physics. To demonstrate the advantages of using a distribution with bounded support, we compared trajectories planned using moments of GMM predictions with trajectories planned using truncated GMMs with the same mean and covariance. While the distributions' means and covariances match, their higher order moments differ significantly as the kurtosis, a metric for the ``heaviness" of tails, is expressed by the fourth order moment. These higher order moments are captured in our formulation because the variance of the random variable $Q(\mathbf{a}_t)-1$ is a function of moments of the prediction up to order four. By leveraging this additional moment information, the planner can plan less conservative trajectories with truncated GMMs that are still safe; figure \ref{fig:truncated_comparison} shows an example.
\begin{figure}[!h]
    \centering
    \includegraphics[width=\linewidth]{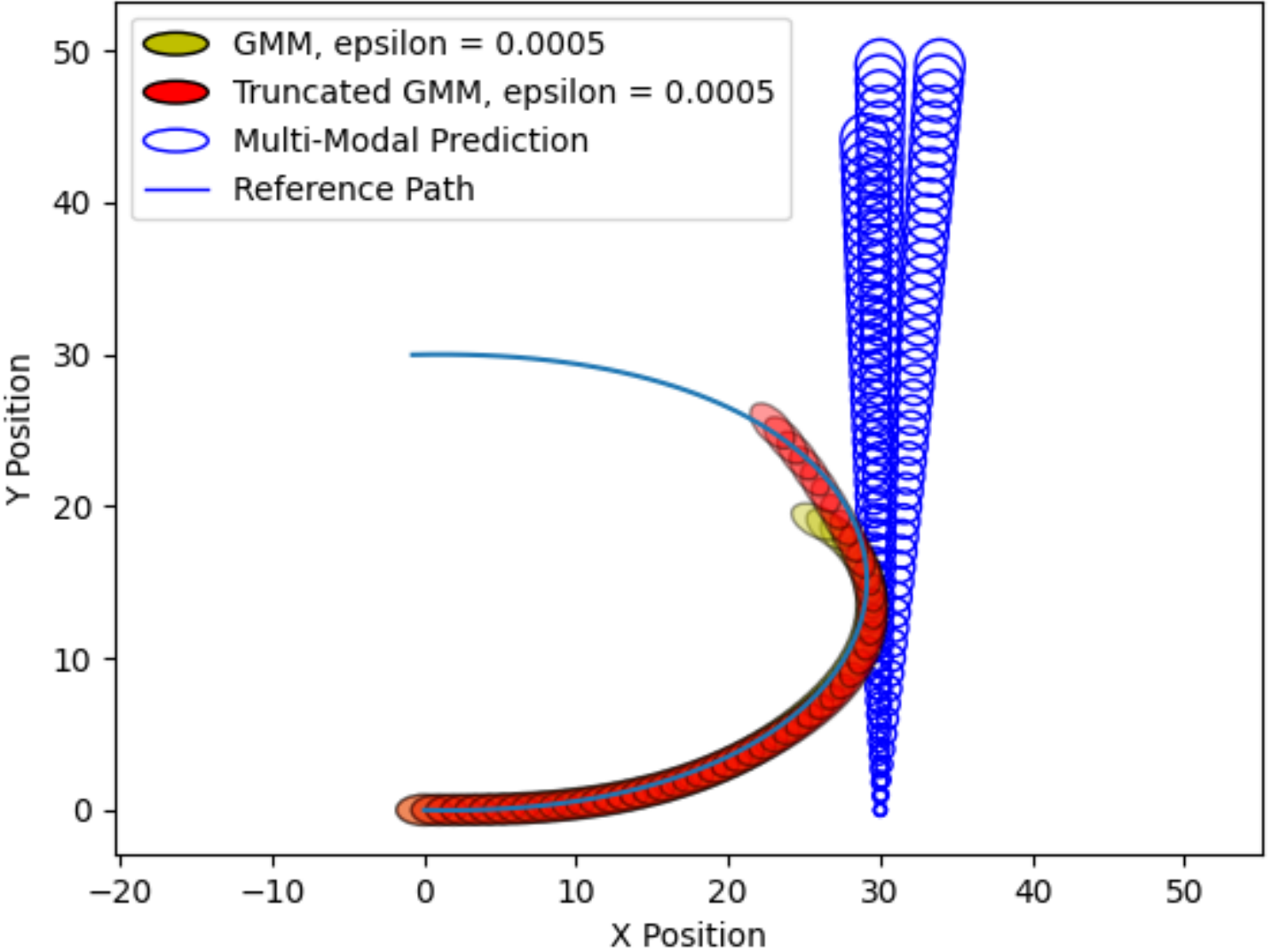}
    \caption{Planned trajectories from using a GMM and a truncated GMM with the same mean vector and covariance.}
    \label{fig:truncated_comparison}
\end{figure}

To test at a larger scale, 100 different U-turn scenarios were generated by perturbing the parameters of the polynomial path and 10 chance-constraint settings were tested for each formulation on each scenario, totaling 2000 optimizations. As shown in figure \ref{fig:compare_results}, the truncated GMM formulation solves the problem more reliably than the GMM formulation and also produces lower cost trajectories. Over 2000 optimizations, the average solution time was 12.9 ms, making this formulation fast enough for use online.
\begin{figure}[!h]
    \centering
    \includegraphics[width=\linewidth]{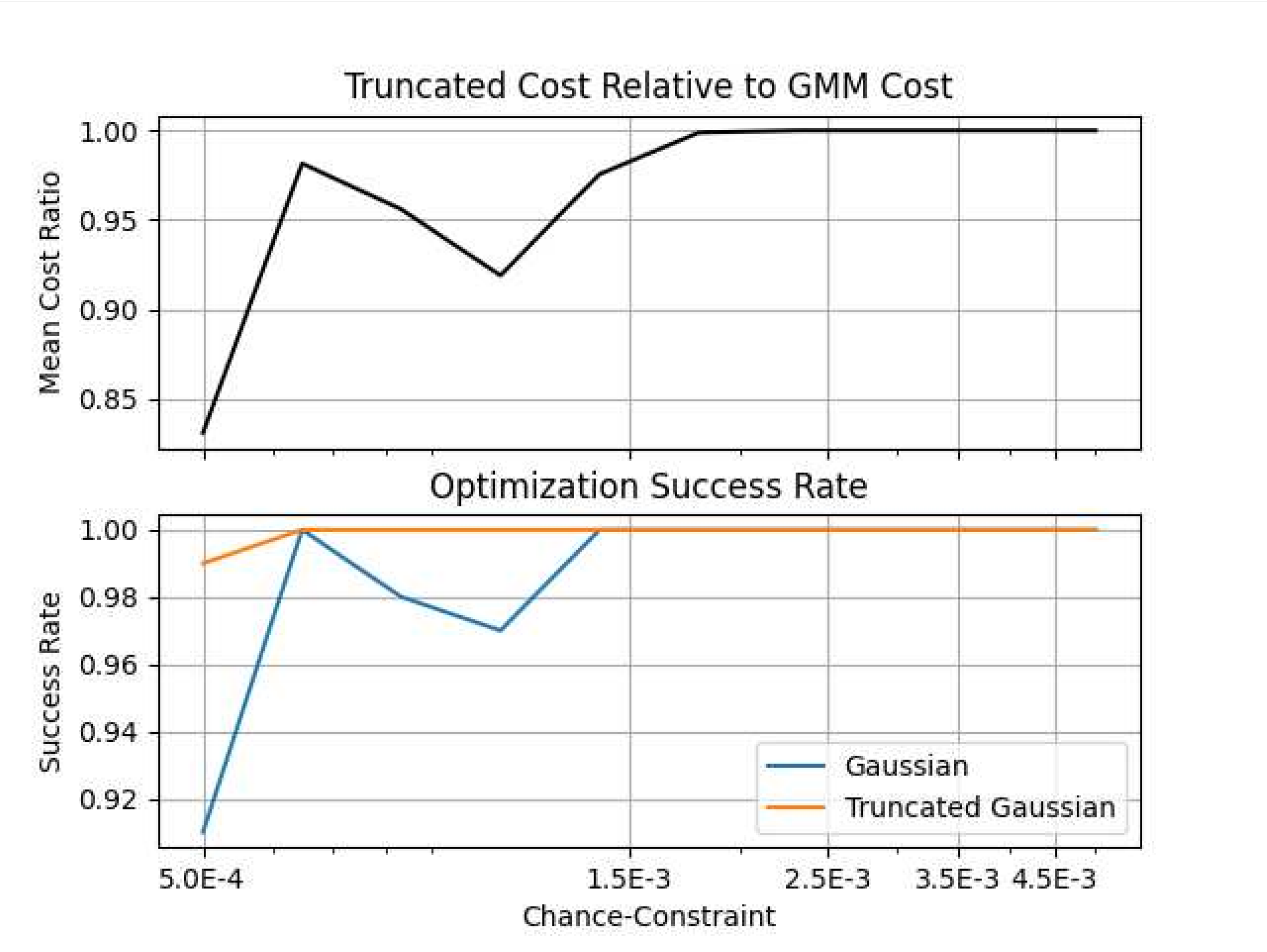}
    \caption{Plots of the mean cost ratio and solution success rates across the range of chance-constraints tested.}
    \label{fig:compare_results}
\end{figure}

\section{CONCLUSIONS}
In this paper, we presented a chance-constrained trajectory planning formulation for autonomous vehicles that can handle mixtures of non-Gaussian distributions of agent positions and does not need to make point mass assumptions. To enforce the chance-constraint, we present a general framework that leverages symbolic algebra to generate expressions that upper-bound polynomial chance-constraints in terms of statistical moments of the underlying distribution. We demonstrated this approach with numerical experiments and show the optimization problem is fast enough for use online. We also provide a Python package, AlgebraicMoments, to enable other members of the community to adopt our approach to bounding chance-constraints. We note that while enforcing chance-constraints using the presented concentration inequalities provides great generality, this generality can also produce excessively conservative results. Future works should consider leveraging more distribution specific information to establish tighter bounds.

\section{Appendix}
\begin{lemma}\label{lemma:convex_func}
The function $\phi(x, y) = \frac{x^2}{y}$  on the domain $x\in\R$ and $y > 0$ is convex.
\end{lemma}
\begin{proof}
It is sufficient for the Hessian of $\phi$ to be positive semi-definite (psd). The eigenvalues of the Hessian can be found in closed form with symbolic algebra:
\begin{align}
    \lambda_1 = 0\quad \lambda_2 = 2(x^2 + y^2)y^{-3}
\end{align}
Since $y > 0$ on the domain of $\phi$, both eigenvalues are non-negative, so the Hessian is psd on the domain of $\phi$.
\end{proof}
\subsection{Proof of Theorem \ref{thm:cheb_mms}}
\begin{proof}
We begin by showing the inequality (\ref{eq:decompose_cheb}). The probability of constraint violation can be written as the expectation of the indicator function. Define the indicator function:
\begin{align}
    \mathds{1}_{(-\infty, 0]}(x) = \begin{cases}
    1 \quad x\in (-\infty, 0]\\
    0 \quad o.w.
    \end{cases}\label{eq:indicator}
\end{align}
The probability can then be expressed as:
\begin{align}
    \Prob(g(\mathbf{w})\leq 0) &= \mathbb{E}[\mathds{1}_{(-\infty, 0]}(g(\mathbf{w}))]
\end{align}
The indicator function (\ref{eq:indicator}) is measurable \cite{hunter2011measure}, so we can apply Proposition \ref{prop:decompose_exp_g}:
\begin{align}
    \mathbb{E}[\mathds{1}_{(-\infty, 0]}(g(\mathbf{w}))]
    &= \sum_{i=1}^n w_i\Prob(g(\mathbf{w}_i)\leq 0)\\
    &\leq \sum_{i=1}^n w_i\text{Conc}[g(\mathbf{w}_i)]
\end{align}
Now, we show the inequality (\ref{eq:mm_thm_leq}) for the Cantelli and VP cases. Since VP is just a constant scaling of Cantelli, the results for Cantelli apply directly to VP. We rewrite:
\begin{align}
    \frac{\text{Var}[g(\mathbf{w})]}{\text{Var}[g(\mathbf{w})] + \mathbb{E}[g(\mathbf{w})} &= 1 - \frac{\mu_{g(\mathbf{w})}}{\mathbb{E}[g(\mathbf{w})^2]}\\
    &= 1 - \phi\left(\mu_{g(\mathbf{w})}, \mathbb{E}[g(\mathbf{w})^2]\right)\label{eq:phi}
\end{align}
By Lemma \ref{lemma:convex_func}, $\phi$ above is convex since $\mathbb{E}[g(\mathbf{w})^2] > 0$ \footnote{Technically, it is possible for the second moment to be zero when the random variable is zero with probability one, but this is a pathological case that should never be encountered in practice.}. By decomposing the moments we have:
\begin{align}
    \begin{bmatrix}
    \mu_{g(\mathbf{w})}\\
    \mathbb{E}[g(\mathbf{w})^2]
    \end{bmatrix} = \sum_{i=1}^n w_i\begin{bmatrix}
    \mu_{g(\mathbf{w}_i)}\\
    \mathbb{E}[g(\mathbf{w}_i)^2]
    \end{bmatrix}
\end{align}
Thus, by the finite version of Jensen's inequality:
\begin{align}\label{eq:jensen}
  \phi\left(\mu_{g(\mathbf{w})}, \mathbb{E}[g(\mathbf{w})^2]\right) \leq \sum_{i=1}^n w_i\phi(\mu_{g(\mathbf{w}_i)}, \mathbb{E}[g(\mathbf{w}_i)^2])
\end{align}
Subtracting the left hand side quantity from $1$ and the right hand side quantity from $\sum_{i=1}^n w_i = 1$, we have:
\begin{align}
    \text{Conc}[g(\mathbf{w})]&\geq \sum_{i=1}^n w_i(1-\phi(\mu_{g(\mathbf{w}_i)}, \mathbb{E}[g(\mathbf{w_i})^2]))\\
    &= \sum_{i=1}^n w_i\text{Conc}[g(\mathbf{w}_i)]
\end{align}
The exact same argument can be applied to the Gauss inequality by establishing convexity of:
\begin{align}
    \text{Var}[g(\mathbf{w}_i)]/\mu_{g(\mathbf{w}_i)}^2
\end{align}
on the domain with $\text{Var}[g(\mathbf{w}_i)] > 0$. We now turn our attention to establishing almost sure strictness of the inequality. It will be sufficient to show that Jensen's inequality is strict. Jensen's inequality is strict if $\phi$ is strictly convex. The key idea is that $\phi$ is indeed strictly convex if we further restrict its domain by removing a set of Lebesgue measure zero. Recall from Lemma \ref{lemma:convex_func} that on the domain $D$ of $\phi$, the eigenvalue $\lambda_2$ is strictly positive, but $\lambda_1=0$. Thus, if we restrict the domain of $\phi$ to not contain the eigenspace of $\lambda_1=0$ for each component, denote this restricted domain $\hat{D}$, then we have that $\phi$ is strictly convex on $\hat{D}$. This set $\hat{D}$ is characterized in terms of moments of components of $g(\mathbf{w})$. The eigenspace for the $i_{th}$ component, $\hat{D}_i$ is:
\begin{align}
    \mathbf{u}_i &:= \left[\frac{\mu_{g(\mathbf{w}_i)}}{\mathbb{E}[g(\mathbf{w}_i)^2]}, 1\right]^T\quad
    \hat{D}_i &:= \{\alpha \mathbf{u}_i : \alpha\in\R\}
\end{align}
We arrive at $\hat{D}$ by removing the union of $\hat{D}_i$:
\begin{align}
    \hat{D}:= D - \cup_{i=1}^n\hat{D}_i
\end{align}
Note that $\cup_{i=1}^n\hat{D}_i$ is the union of lines and, thus, has Lebesgue measure zero. One complication is that $[\mu_{g(\mathbf{w}_i)}, \mathbb{E}[g(\mathbf{w}_i)^2]]^T$ does lie on $\hat{D}_i$, but it will be sufficient for it to not lie on any other $\hat{D}_i$; that is, we require:
\begin{align}
 \forall i\in[n], \quad [\mu_{g(\mathbf{w}_i)}, \mathbb{E}[g(\mathbf{w}_i)^2]]^T\notin \cup_{j\in[n], j\neq i}\hat{D}_j\label{eq:component_restriction}
\end{align}
By requiring the first and second moments of each component to not lie on any $\hat{D}_i$ line corresponding to other components, we require Jensen's inequality to cross the interior of the set on which $\phi$ is strictly convex. By the analytic statement of Jensen's Inequality under strict convexity \cite{jensen_analysis}, we thus have that under the conditions $w_i>0$ and $[\mu_{g(\mathbf{w}_i)}, \mathbb{E}[g(\mathbf{w}_i)^2]]^T$ are unique $\forall i\in[n]$:
\begin{align}
  \phi\left(\mu_{g(\mathbf{w})}, \mathbb{E}[g(\mathbf{w})^2]\right) < \sum_{i=1}^n w_i\phi(\mu_{g(\mathbf{w}_i)}, \mathbb{E}[g(\mathbf{w}_i)^2])
\end{align}
Since the set $\cup_{i=1}^n\hat{D}_i$ has measure zero, if we choose first and second moment pairs at random according to any distribution supported on some subset of $\R^2$ with non-zero Lebesgue measure, the probability of the components not satisfying the condition $(\ref{eq:component_restriction})$ is zero.
\end{proof}
\newpage
\subsection{Cost Parameters}

\begin{table}[!h]
\centering
\begin{tabular}{|l|l|l|}
\hline
\textbf{Parameter} & \textbf{Description}       & \textbf{Value}                                                                                    \\ \hline
$c_D$              & Contouring deviation cost  & 20                                                                                                \\ \hline
$c_L$              & Contouring lag cost        & 1                                                                                                 \\ \hline
$c_v$              & Speed deviation cost       & 1                                                                                                 \\ \hline
$R$                & Control effort cost matrix & $\begin{bmatrix}

1 & 0.0\\

0.0 & 100

\end{bmatrix}$\\ \hline
\end{tabular}
\vspace{1.5mm}
\caption{Table of Cost Parameters}
\end{table}
\bibliographystyle{IEEEtran}
\bibliography{references}

\end{document}